%% file: CSLearning-Arxiv.tex
\documentclass[11pt]{article}
\usepackage{hyperref,url}
\usepackage{multirow,array}
\usepackage{amsfonts,amsmath,amssymb,amsthm}
\usepackage{rotating,graphicx}
\usepackage{mathptmx,booktabs}
\usepackage{pdflscape}
\usepackage{authblk}
\usepackage[misc]{ifsym}
\usepackage{enumitem}
\setlist{itemsep=0pt}
\usepackage[height=25cm,width=18.75cm]{geometry}
\newtheorem{theorem}{Theorem}
\newtheorem{remark}{Remark}
\usepackage{tikz,xcolor,hyperref}

\definecolor{lime}{HTML}{A6CE39}
\DeclareRobustCommand{\orcidicon}{
\hspace{-3mm}	
	\begin{tikzpicture}
	\draw[lime, fill=lime] (0,0) 
	circle [radius=0.16] 
	node[white] {{\fontfamily{qag}\selectfont \tiny ID}};
	\draw[white, fill=white] (0.0,0.0) 
	circle [radius=0.000];
	\end{tikzpicture}
	\hspace{-3mm}
}

\foreach \x in {A, ..., Z}{%
	\expandafter\xdef\csname orcid\x\endcsname{\noexpand\href{https://orcid.org/\csname orcidauthor\x\endcsname}{\noexpand\orcidicon}}
}

\title{Misclassification Cost-sensitive Ensemble Learning:\\ A Unifying Framework} 
\author[1,2]{George Petrides\thanks{The main body of this work was finished while this author's main affiliation was VUB. Email: firstname.lastname@uib.no}\orcidA{}(\Letter)}
\author[2]{Wouter Verbeke\orcidB{}
}
\affil[1]{University of Bergen, Norway}
\affil[2]{Vrije Universiteit Brussel (VUB), Belgium}
\date{\empty}
\begin{document}
\maketitle

\begin{abstract}
Over the years, a plethora of cost-sensitive methods have been proposed for learning on data when different types of misclassification errors incur different costs. Our contribution is a unifying framework that provides a comprehensive and insightful overview on cost-sensitive ensemble methods, pinpointing their differences and similarities via a fine-grained categorization. Our framework contains natural extensions and generalisations of ideas across methods, be it AdaBoost, Bagging or Random Forest, and as a result not only yields all methods known to date but also some not previously considered.
\end{abstract}
\input{SecIntro}
\input{SecPrelim}
\input{SecCSLearn}
\input{SecCSClass}
\input{SecApplic}
\input{SecConc}
\section*{Acknowledgements} This work was supported by Secur'IT, the platform dedicated to information security launched by Innoviris, the Brussels Region Research funding agency. More info on www.securIT.brussels.
The authors would like to thank some of the anonymous reviewers of earlier versions of this work for their useful recommendations.
\bibliographystyle{plain}
\bibliography{CSLearning} 
\newpage
\input{Appendix}
\end{document}

%% file: SecIntro.tex
\section{Introduction}\label{sec:intro}
The task of supervised machine learning is given a set of recorded observations and their outcomes to predict the outcome of new observations. Standard classification techniques aim for the highest overall accuracy or, equivalently, for the smallest total error, and include among others support vector machines, Bayesian classifiers, logistic regression, decision tree classifiers such as CART \cite{cart} and C4.5 \cite{C45}, and  ensemble methods which build several classifiers and aggregate their predictions such as Bagging \cite{bagging}, AdaBoost \cite{AdaB} and Random Forests \cite{RF}.

Of particular interest in certain domains are \textit{binary} classifiers which deal with cases where only two classes of outcomes are considered, such as fraudulent and legitimate credit card transactions, responders and non-responders to a marketing campaign, patients with and without cancer, intrusive and authorised network access, and defaulting and repaying debtors to name a few. In most of these cases, one of the classes is a small minority and consequently traditional classifiers might classify all of its members as belonging to the majority class without any significant overall accuracy loss. The severity of this \textit{class imbalance} becomes more noticeable when failing to correctly predict a minority class member is more costly than doing so with a member of the majority class, as the case often is.

A remedy to the undesirable situation just described are classifiers which, instead of accuracy, take misclassification costs into account and are thus termed \textit{cost-sensitive}. We illustrate this idea in the credit card fraud detection framework: accepting a fraudulent transaction as legitimate incurs a cost equal to its amount. Conversely, requiring an additional security check for a transaction (such as contacting the card owner) incurs an overhead cost. The job of a cost-sensitive classifier is to find a cost-minimising balance between overhead costs and fraud costs.
\subsection{Related Work}\label{sec:RelatedWork}
An increased research interest in cost-sensitive learning that spanned more than a decade was witnessed in the mid-nineties,
including among others \cite{Knoll,rcppruning,cspruning,wtree,CSB0,metacost,AdaC,EvalMetaCost, CostTypes,foundations,DCSDM,AsymAB,wtreejour,costing,DTMC,thresholding,AdaC123}.
In recent years, a renewed interest is observed 
\cite{cous,CSDT,NeedCSAda}, 
partly attributed to practitioners starting to realise the potential of using cost-sensitive models for their businesses by considering real-world monetary costs. 

The most recent articles providing an overview of 
these methods differ in the fraction of the classifier spectrum they cover. 
Often, cost-sensitive learning is reviewed within surveys on learning on imbalanced datasets as an approach towards treating class imbalance in any domain by artificially introducing costs  \cite{surHe,surPrati,surSun,surGalar}. Employing a 
cost minimisation point of view, \cite{surLomax} reviews cost-sensitive classifiers based on decision trees. Overall, cost-sensitive boosting methods receive more attention than other methods such as weighting, altered decisions and cost-sensitive node splitting. 
%
\subsection{Our Contribution}\label{sec:contributions}
Our primary contribution in this article is a unifying framework of binary ensemble classifiers that, by design or after slight modification,  are cost-sensitive with respect to misclassification costs. It is presented in terms of combinable components that are either directly extracted from the existing literature or indirectly via natural extensions and generalisations we have identified. A notable example of such an extension are ways in which costs can influence the aggregation of the outputs of the individual models in any ensemble, as done in AdaBoost (Sect.~\ref{sec:voting}). As such, our work goes one step further than being a mere survey. 
The advantages of our approach include that 
\begin{enumerate}[label=(\alph*)]
\item by abstracting the core ideas behind each classifier, we are able to provide generic descriptions that allow for a fine-grained categorisation with respect to the way costs influence the final decision,
\item it makes the similarities and differences between methods easier to recognise (see for example Table~\ref{table:CostBoost} and the equivalence proven in Thm.~\ref{thm:mec-mtaEquiv}),
\item it clearly indicates the types of costs (constant or record-dependent) that are applicable for each method,
\item combining the framework components in all possible ways not only yields all methods known to date, but also some not previously considered (see for example Tables~\ref{table:modeloverview} and \ref{table:algoverview}),
\item  the framework components are generic enough to be instantiated with different classifiers, including Random Forests (see for example Table~\ref{table:algoverview}), and 
\item it highlights research directions that can lead to new cost-sensitive methods (see Sect.~\ref{sec:FutureWork}).
\end{enumerate}
\subsection{Outline}
We give a brief introduction to decision tree classifiers, ensemble methods and cost-sensitive learning in Sect.~\ref{sec:preliminaries}, before presenting our framework of cost-sensitive components in Sect.~\ref{sec:CSC}. In 
Sect.~\ref{sec:experiments} we discuss the road towards the state of the art, and we end the paper with our conclusions and directions for future work in Sect.~\ref{sec:conclusions}.

%% file: SecPrelim.tex
\section{Preliminaries}\label{sec:preliminaries}
Most of this section's material is provided with the intention of making the article as self contained as possible. We begin with the basics of decision tree classifiers which play a central role in this work. Readers familiar with these and ensemble methods can proceed to Sect.~\ref{sec:CSL} for an introduction to cost-sensitive learning.

A dataset is a collection of records which consist of a number of characteristics, often referred to as \textit{attributes} or \textit{features}. A record's outcome or \textit{class} is what is of importance and needs to be predicted.
Classifiers are trained using a set of records together with their known class in order to be able to predict the class of other records for which it is unknown.

In this work our interest lies in the binary class case where there are only two possibilities for the class.
In binary imbalanced datasets, it is customary to call records within the minority class  \textit{positive} and within the majority class \textit{negative}. Throughout this article, the class of positive records will be denoted by $1$ and that of negative ones by $0$. 

Distinction is made between the different predictions a classifier makes. \textit{True Positive} (TP) and \textit{False Negative} (FN) denote a positive record correctly classified and misclassified respectively, and \textit{True Negative} (TN) and \textit{False Positive} (FP) are the equivalents for a negative record.
\subsection{Decision Tree Classifiers}
Decision Tree classifiers are greedy algorithms that try to partition 
datasets according to their records' outcomes 
through a series of successive splits. For each split, the attribute that partitions the records the best according to some metric is chosen, and splitting ends after each partition contains records of only one class, or when further splits do not improve the situation. 

Starting from the initial set of all records, or the tree's \textit{root node}, splits create \textit{branches} in the tree which are labelled by the attribute used for splitting. Split sets are known as \textit{parent nodes}, sets obtained after a split are called \textit{children nodes}, and sets that are no longer split are called \textit{leaf nodes}. 

Usually, the next step after tree growing is \textit{pruning}, done by removing nodes which do not improve accuracy. Pruning is a process starting from the bottom of the tree and going up and serves the purpose of reducing \textit{over-fitting}, that is the effect of the tree's quality of predictions not generalising beyond the dataset used for training.  

In the final tree, each leaf node is assigned the class with the highest frequency among its records, and every record reaching the node will be predicted as having that class.

For certain parts of a decision tree algorithm, such as the node splitting step, it is necessary to know the probabilities that a record reaching a node $t$ is positive ($P_{t_+}$) and negative ($P_{t_-}$). Since $P_{t_-}=1-P_{t_+}$, it in fact suffices to know one of the two. Let $N_t$, $N_t^+$ and $N_t^-$ respectively denote the sets of all, the positive, and the negative records at node $t$ (when no subscript is specified we will be referring to the root node of the tree, or the set of all records). Then, 
$$P_{t_+}=|N_t^+| / |N_t| \enspace.$$
As probabilities based on frequency counts can be unreliable due to high bias and variance, they can be calibrated using methods like \textit{Laplace Smoothing} as suggested in \cite{rcppruning} 
 ($P_{t_+}=\left (|N_t^+|+1 \right ) / \left (|N_t|+2 \right )$), \textit{m-estimation} \cite{m-est,DCSDM,Elkancalib} ($P_{t_+}=\left( |N_t^+|+b\cdot m \right) /\left( |N_t|+m \right)$, where $b=|N^+| / |N|$ and $b\cdot m \approx 10$),  and \textit{Curtailment} \cite{DCSDM,Elkancalib} (each node with less than $m$ records, $m$ as before, gets assigned the probability of its closest ancestor with at least $m$ records) and combinations of the latter with any of the former two.
 
Examples of decision tree classifiers include the widely used CART \cite{cart} and C4.5 \cite{C45}, briefly described below.
\subsubsection{CART} CART (\textit{Classification and Regression Trees}, \cite{cart}) uses the Gini index as a splitting measure during tree growing. More specifically, the attribute chosen for splitting is the one that maximises the following value, known as \textit{gain}:  
\begin{equation}\label{eq:cartgain}
1-P_{t_+}^2-P_{t_-}^2 - \displaystyle \sum_{i=1}^k\frac{|N_{t_i}|}{|N_t|} \left (1-P_{t_{i+}}^2-P_{t_{i-}}^2 \right )\enspace,
\end{equation}
 where $t_1$  to $t_k$ are the children nodes of tree node $t$.

The pruning method used by CART is \textit{cost complexity pruning}. Given a tree $T$ and $\alpha \in \mathbb{R}$ the aim is to find the subtree of $T$ with the smallest error approximation, which is its error on the training data (the number of wrong predictions over the number of correct ones) plus $\alpha$  times the number of its leaf nodes. Starting from the tree to be pruned and $\alpha=0$, a finite sequence of subtrees and increasing $\alpha$s is obtained in this way. The tree chosen as the pruned tree is the one with the smallest error approximation on a separate validation set (a set not used for anything else). 

By design, CART can take record weights $w$ as input during training and use them to modify the probabilities used in calculating the gain (\ref{eq:cartgain}) as
$$P_{w_{t_+}}=W_t^+ / W_t\enspace,$$ where  $W_t^+$ and $W_t$ respectively denote the sum of weights of all positive and all records at node $t$. Moreover,  minimum weight is considered instead of minimum error for pruning, and each leaf node is assigned the class with the largest total weight among its records.
\begin{remark} It is not clear if and how weighted probabilities can be calibrated. \end{remark}
\subsubsection{C4.5} 
The splitting measure of C4.5 \cite{C45} is an extension of Entropy which is a normalised version known as \textit{gain ratio}:
$$
\displaystyle \frac{P_{t_+}\log_2{P_{t_+}} + P_{t_-} \log_2{P_{t_-}} - \sum_{i=1}^k\frac{|N_{t_i}|}{|N_t|} \left( P_{t_{i+}}\log_2{P_{t_{i+}}} + P_{t_{i-}} \log_2{P_{t_{i-}}} \right)
}{\sum_{i=1}^k\frac{|N_{t_i}|}{|N_t|}\log_2{\frac{|N_{t_i}|}{|N_t|}}} \enspace,
$$
where $t_1$  to $t_k$ are the children nodes of tree node $t$.
 C4.5 employs \textit{reduced error pruning}.

 Ting \cite{wtree,wtreejour} showed how to implement the weighted CART design in C4.5.
\subsection{Ensemble Methods}
In ensemble methods, several models are trained and their outcomes combined to give the final outcome, usually as a simple or weighted majority vote, the difference lying on whether each model's vote weighs the same 
 (as in Bagging and Random Forests) or may weigh differently (as in AdaBoost). 
Probabilities are usually combined by taking their average.

Some of the most important ensemble methods are briefly described below.
\subsubsection{Bagging}\label{sec:BG}
The idea of Bagging  \cite{bagging} is to build several models (originally CART models, though in principle there is no restriction) on samples of the data. If the sampled sets are of equal size as the original data, they are called \textit{bootstraps}.
\begin{enumerate}
\item Sample with replacement a number of uniformly random and equally sized sets from the training set.
\item\label{bag:model} For each sampled set, build a model producing outcomes or probabilities.
\item A record's final outcome (respectively probability $P_+$) is the the majority vote on its outcome (respectively the average of its probabilities) from all models.
\end{enumerate}
\subsubsection{Random Decision Forests}\label{sec:RS}
As originally defined in \cite{rdf,randomsub}, Random Decision Forests differs from Bagging in that it samples subsets of the attribute set instead of the data to build decision trees. Here we will abuse terminology slightly and consider Random Decision Forests as a special case of Bagging that builds \textit{Random Feature Trees}, a name we give to decision trees that are grown on a random subset of the attribute set.
\subsubsection{Boosting} \label{sec:AdaB}
Boosting refers to enhancing the predictive power of a "weak" classifier by rerunning it several times, each time focusing more on misclassified records.  

AdaBoost \cite{AdaB} is the most notable example of Boosting in which the focus on each record is in terms of weights: misclassified records after a round get increased weights and correctly classified ones get decreased weights.
\begin{enumerate}
\item \label{item:abwts} Assign weight $w=1$ to each record in the training set.
\item\label{item:normalise} Normalise each record's weight by dividing it by the sum of the weights of all records. 
\item \label{item:abmodelbuilt}Build a model using the weighted records and obtain each record's outcome $h \in \{0,1\}$
  and the model's total error $\epsilon$ as the sum of weights of all misclassified records.
\item \label{item:abprobrep}Update each record's weight as $w'=w\cdot e^{-\alpha y_* h_*}$, where $\alpha=\frac{1}{2}\ln\left( \frac{1-\epsilon}{\epsilon}\right)$,
and $h_*$ and $y_*$  are \textit{h} and \textit{y}, the record's true class, mapped from $\{0,1\}$ to $\{-1,1\}$.
\item\label{item:abrepeat} Repeat steps \ref{item:normalise} to \ref{item:abprobrep} as required.
\item\label{item:adaboostfinal} A record's final outcome is the weighted majority vote on its outcome from all models,  the weights being the $\alpha$s.
\end{enumerate}
A generalised version of AdaBoost with different $\alpha$ was proposed in \cite{GenAB}. In \cite{AdaProb} methods were investigated for obtaining reliable probabilities from AdaBoost, something that by default it is incapable of doing, through calibrating $S$, the normalised sum of weighted model votes. These are 
\textit{Logistic Correction} \cite{LogCorr} ($P_{lc}=1/(e^{-2\left(2\cdot S-1\right)}+1)$, where the name was coined in \cite{AdaProb}), \textit{Platt Scaling} \cite{Platt,IsoReg2} ($P_{ps}=1/(e^{A\cdot S+B}+1)$, where A and B maximise $P_{ps}$ on a validation set with classes mapped from $\{0,1\}$ to $\{(|N^+|+1)/(|N^+|+2),1/(|N^-|+2)\}$) and \textit{Isotonic Regression} \cite{IsoReg1,IsoReg2} (essentially an application of the PAV algorithm \cite{PAV}:
 (1) sort training records according to their sum $S$, (2) initialise each record's probability as 0 if negative and 1 if positive, (3) whenever a record has higher probability than its successor, replace the probability of both by their average and consider them as one record thus forming intervals, and (4) a record's probability is the one of the interval its sum $S$ falls in).
\subsubsection{Random Forests}\label{sec:RF}
Random Forests \cite{RF} is in fact Bagging confined to tree classifiers with \textit{Random Input Selection}, which at each splitting step choose the best attribute out of a small randomly chosen subset of all attributes,  and are not pruned.

%% file: SecCSLearn.tex
\subsection{Cost-Sensitive Learning}\label{sec:CSL}
Despite the absence of a formal definition, the informal consensus is that \textit{cost-sensitive (CS) learning} refers to aiming at minimising costs related to the dataset instead of error, typically via these costs influencing the classification process in some way.
In this work we only consider misclassification costs, though other types exist, such as the cost of attribute acquisition and obtaining attribute values that are missing.

Traditionally, different costs (or benefits) assigned to each type of classification are given in the form of a \textit{Cost Matrix}: 
\begin{center}
$CM=\left [
\begin{array}{cc}
C_{TP} \hspace{2mm} & C_{FN} \vspace{2mm}\\
C_{FP} \hspace{2mm} & C_{TN}\\
\end{array}
\right ]$
\end{center}
In the sequel, we only consider misclassification costs that are higher than costs of correct classification, and by letting $C_{TN}'=C_{TP}'=0$, $C_{FP}'=C_{FP}-C_{TN}$  and $C_{FN}'=C_{FN}-C_{TP}$, we can reduce our attention to only misclassification costs, even when the other costs are non-zero \cite{foundations}. 
 
Costs can be either constant for all records of a class, often called \textit{class-dependent}, or vary per record which we will call \textit{record-dependent}. 
For instance, in credit card fraud detection, false positive costs $C_{FP}$ are equal to overhead costs and can be  the same for all transactions, 
whereas false negative costs $C_{FN}^i$ depend on the individual transactions $i$ and are equal to the corresponding amount. 
\subsubsection{CS Decisions}\label{sec:csdecisions}
 A cost-insensitive classifier would label a record as positive if $P_+ > P_-$ or equivalently if $P_+ > 0.5$.
 As explained by Elkan  \cite{foundations}, this decision can be made cost-sensitive by using the minimum expected cost (MEC) criterion, that is by labelling a record as positive if 
 $ C_{FP}\cdot P_-  <  C_{FN} \cdot P_+$,
or equivalently if $P_+ > T_{cs}$, where $T_{cs}$ is the cost-sensitive threshold 
\begin{equation} \label{eq:SCST}
 T_{cs}=\frac{C_{FP}}{C_{FP}+C_{FN}} \enspace.
\end{equation}
Note that $T_{cs}=0.5$ corresponds to the case of equal misclassification costs, $C_{FN}>C_{FP}$ implies $T_{cs} <0.5$ and  $C_{FN}<C_{FP}$ implies $T_{cs} > 0.5$.  
\begin{remark}
The case of record-dependent costs can be treated by considering a distinct threshold $T^i_{cs}$ per record \textit{i}, as first observed in \cite{Elkancalib}.
\end{remark}
\textit{Thresholding} \cite{thresholding}, instead of using the theoretical threshold (\ref{eq:SCST}), looks for the best threshold $T_{thr}$ among all probabilities obtained from the training set by computing the total costs for each on a validation set and choosing the one with the lowest. 
\subsubsection{CS Data Sampling}\label{sec:cssampling}
To induce decision making using the threshold $T_{cs}$ of (\ref{eq:SCST}) instead of 0.5 when the cost ratio is constant, we can \textit{under-sample} the negative training records by only sampling $|N^-|\cdot
 \frac{C_{FP}}{C_{FN}}$ 
out of $|N^-|$ \cite{foundations}. Equivalently, we can \textit{over-sample} the positive training records by duplicating existing ones or by synthesising new records to reach a total of $|N^+|\cdot \frac{C_{FN}}{C_{FP}}$
instead of $|N^+|$. Sampling and duplicating can either be random or targeted according to some rule. Naturally, any combination of these techniques that yields a positive-negative ratio equal to 
\begin{equation}\label{eq:csratio}
r_{cs}=\frac{|N^+|}{|N^-|} \cdot \frac{C_{FN}}{C_{FP}} 
\end{equation} is possible, and we shall call it \textit{hybrid-sampling}.
\begin{remark}If  $\frac{C_{FN}}{C_{FP}} > \frac{|N^-|}{|N^+|}$ then sampling turns the positive class into the majority. Under-sampling reduces the size of training data and consequently model training time at the cost of losing potentially useful data. On the other hand, over-sampling makes use of all data but leads to increased training times, and record duplication entails the risk of over-fitting.
\end{remark}
%
One method for synthesising new records, thus avoiding the risk of over-fitting is the \textit{Synthetic Minority Oversampling Technique (SMOTE)} \cite{SMOTE}, which over-samples positive records by creating new ones that are nearest neighbours (roughly speaking, that have the closest similarity attribute-wise) to existing ones:
\begin{enumerate}
\item Choose a positive record and find some (say \textit{k}) of its nearest neighbours.
\item For each nearest neighbour, find its per attribute distance $d_a$ with the positive record. 
\item Create a new positive record with attributes those of the positive record minus a random fraction of $d_a$. 
\item Repeat as required, keeping \textit{k} fixed.
\end{enumerate}
%
An alternative to sampling or synthesising records to reach the ratio $r_{cs}$ in (\ref{eq:csratio}) is \textit{Cost-Proportionate Rejection (CPR) Sampling} \cite{costing}, which is also applicable when costs are record-dependent, and where a sampled record is accepted with probability proportional to its cost:
\begin{enumerate}
\item Sample with replacement a uniformly random set from  the training set.
\item Create a new training set that includes each of the sampled  set's elements with probability $C_{FN}/\max\{C_{FN},C_{FP}\} $ if positive or  $C_{FP}/\max\{C_{FN},C_{FP}\}$ if negative.
\end{enumerate} 
\subsubsection{CS Record Weights}
Ting \cite{wtree} was the first to explicitly incorporate costs in the weights $w_+$ of the positive and $w_-$ of the negative classes used in weighted classifiers, followed by normalisation:
\begin{equation}\label{eq:posnegweight} 
w_+= C_{FN} \mbox{ \hspace{2mm} and \hspace{2mm} }  w_-= C_{FP} \enspace. 
\end{equation}
\begin{remark}
We observe that record-dependent costs can be easily taken into account by replacing $C_{FN}$ by $C^i_{FN}$ and $C_{FP}$ by $C^i_{FP}$. Clearly, equal costs yield equal weights.
\end{remark}

%% file: SecCSClass.tex
\section{Cost-Sensitive Ensemble Methods}\label{sec:CSC}
\begin{figure}[!t]
\input{FigAlgCat}
\caption{The main categorisation of cost-sensitive ensemble methods with respect to the point that costs influence the classification process. For a non-exhaustive list of pre- and during-training methods see Table~\ref{table:algoverview}, and for post-training methods see Table~\ref{table:modeloverview}}
\label{figure:CSCategories}
\end{figure}
Cost-sensitive ensemble methods can be divided into three main categories, depending on when costs influence the classification process:   
before training at the data level (Sect.~\ref{ssec:Pre}), 
during training at the algorithm level (Sect.~\ref{ssec:During}), 
and after training at the decision level (Sect.~\ref{ssec:Post}).
Figure~\ref{figure:CSCategories} provides a summary. Naturally, combinations of these are possible and yield what we shall call \textit{hybrid} methods.

The descriptions we provide are brief and sometimes slightly modified from the original ones in order to unify and generalise them, to incorporate costs when they are not explicitly mentioned, and, where possible, to include record-dependent costs when the original descriptions were given with only constant costs in mind. They are also base-classifier independent, meaning that apart from decision trees that we focus on in this paper, other classifiers such as logistic regression and support vector machines can be used as well.
\subsection{Pre-Training Methods}\label{ssec:Pre}
Pre-training methods employ either cost-sensitive data sampling or record weights, and as a result models need to be retrained in case the costs change. 
\subsubsection{Sampling based}\label{sssec:sampling-based}
\paragraph*{\textbf{CS-SampleEnsemble}}
By this we describe the class of ensembles that use some cost-sensitive sampling method to modify the training set before training each base classifier. This is a generalisation of the concept used in Costing  \cite{costing} where subsets of the training set are obtained by CPR sampling. The other examples found in the literature are  Balanced Random Forest \cite{CSRF}
which uses equally sized sets (originally meant to be balanced and thus cost-insensitive) to build Random Input Selection Tree models (Sect.~\ref{sec:RF}), 
EasyEnsemble \cite{easyens}, an ensemble of ensembles where
under-sampling is used to obtain a number of equally sized sets and build AdaBoost models, and SMOTEbagging and UnderBagging \cite{SMOTEbag}  which respectively use SMOTE and random undersampling. It can be seen as Bagging (Sect.~\ref{sec:BG}) with modified sampling step:
\begin{enumerate}
\item Using a cost-sensitive sampling method, sample a number of sets from the training set.
\end{enumerate} 
\begin{remark}
Although the original definition of Costing did not include models that also produce probabilities, we could not find any reason to exclude them. Costing reduces to Bagging when costs are equal, as CPR-sampling first randomly samples  the data with replacement. Other sampling methods however can sample with replacement at most one of the classes.
\end{remark}
\paragraph*{\textbf{CS-preSampleEnsemble}}
We propose this as the class of ensembles that use some cost-sensitive sampling method to first modify the training set before sampling subsets from it in the manner of Bagging. 
It can be viewed as Bagging (Sect.~\ref{sec:BG}) with the following different steps:
\begin{enumerate}
\item[0.] Modify the training set by means of a cost-sensitive sampling method.
\item Sample with replacement a number of uniformly random and equally sized sets from the modified training set.
\end{enumerate}
\begin{remark} Using CS-undersampling has the disadvantage of producing modified sets of a rather small size. Also, when $C_{FP}=C_{FN}$, CS-preSampleEnsemble reduces to Bagging.  \end{remark}
\paragraph*{\textbf{CS-SampleBoost}}
By this we describe the class of AdaBoost variants that modify the weight of each record by using some sampling method on the training set. The examples found in the literature are SMOTEBoost \cite{SMOTEB} and RUSBoost \cite{RUSB} which respectively use SMOTE and random undersampling. The steps different to AdaBoost (Sect.~\ref{sec:AdaB}) are:
\begin{enumerate}
\item[2a.]  Modify the training set by means of a cost-sensitive sampling method.
\item[2b.]  Normalise each record's weight in the modified set by dividing it by the sum of weights of all records in it.
\item[3.]  Build a model using the modified set of weighted records and obtain for the initial set each record's outcome $h \in \{0,1\}$ and the model's total error $\epsilon$ as the sum of weights of all misclassified records.
\item[5.] Repeat steps~2a to \ref{item:abprobrep} as required.
\end{enumerate}
\begin{remark}When $C_{FP}=C_{FN}$, CS-SampleBoost reduces to AdaBoost.\end{remark}
\subsubsection{Weights based}\label{sssec:weighted}
\paragraph*{\textbf{Naive CS AdaBoost}} Mentioned in \cite{AsymAB,costing} and \cite{CSAB}, its only difference to AdaBoost are the cost-dependent initial weights in Step~\ref{item:abwts}.
\begin{enumerate}
\item Assign to each positive and negative record in the training set weights as in (\ref{eq:posnegweight})
\footnote{In \cite{AsymAB} the costs are actually $w_+=\sqrt{C_{FN}/C_{FP} }$ and $w_-=\sqrt{C_{FP}/C_{FN}}$.}. 
\end{enumerate} 
\begin{remark}When $C_{FP}=C_{FN}$,  Naive CS AdaBoost reduces to AdaBoost.\end{remark}
\paragraph*{\textbf{CS-WeightedEnsemble}}
By this we describe the class of special cases of Bagging where the models built are weighted and the weights initialised as in (\ref{eq:posnegweight}).
It is a generalisation of Weighted Random Forest \cite{CSRF}, a variant of Random Forests that builds weighted Random Input Selection Tree models (see Sect.~\ref{sec:RF}). 
Weighted CART and C4.5 can also be used.
We only describe the steps that are different to Bagging (Sect.~\ref{sec:BG}):
\begin{enumerate}
\item[0.] Assign to each positive and negative record in the training set weights as in (\ref{eq:posnegweight}).
\item[2.] For each sampled set, normalise the weights and build a weighted model.
\item[3.] A record's final outcome is the weighted majority vote on its outcome from all models, the weights being the average record weights at the tree nodes reached by the record.
\end{enumerate}
\subsection{During-Training Methods}\label{ssec:During}
In during-training methods, costs directly influence the way base classifiers are built, which therefore have to be rebuilt when costs change. 
\subsubsection{CS base Ensemble}\label{sssec:csdt}
Cost-insensitive ensemble methods such as Bagging can be made cost-sensitive by employing base classifiers whose decisions based on maximising accuracy are replaced by decisions based on minimising misclassification costs. Restricting our attention in this paper to binary decision trees, the possibilities are cost-sensitive node splitting and/or tree pruning.
\paragraph*{\textbf{CS Node Splitting}} By replacing the impurity measure (such as Entropy or the Gini index) by a cost measure, node splitting in a decision tree is made cost-sensitive. 

An example is Decision Trees with Minimal Costs \cite{DTMC} that do cost-minimising splitting and labelling without pruning. 
In detail, during tree-growing, a node $t$ is labelled according to $P_{t_+}>T_{cs}$. 
The costs $C_t$ at this node are $\sum_{i \in N_{t}^+} C_{FN}^i$ if the node is labelled as negative and $\sum_{i \in N_{t}^-} C_{FP}^i$ otherwise. The attribute selected for node-splitting is the one that instead of maximising the gain value maximises $C_t-\sum_{i=1}^k C_{t_i},$
where $C_{t_i}$ to $C_{t_k}$ are the costs at the children nodes of node $t$, calculated  the same way as $C_t$.
\paragraph*{\textbf{CS Pruning}}{ By replacing the accuracy measure  
 by a cost measure, tree pruning becomes cost-sensitive. Examples include Reduced Cost Pruning \cite{rcppruning} and Cost-Sensitive Pruning \cite{cspruning} which respectively modify the reduced error pruning of C4.5 and the cost-complexity pruning of CART to calculate costs instead of errors, and \cite{Knoll} where both are done.}
\bigskip

A hybrid example are Cost Sensitive Decision Trees \cite{CSDT} that do the same cost-minimising splitting, labelling and pruning mentioned above, with emphasis on record-dependent costs. 
\input{TableCostBoost}
\subsubsection{CS Variants of AdaBoost}\label{ssec:AdaBoostv}
CS variants of AdaBoost typically use the misclassification costs to update the weights of misclassified records differently per class. 
They include UBoost\footnote{Its only difference to its predecessor, \textit{Boosting} \cite{CSB0}, are the cost-based initial weights.} \cite{UBoost}, AdaCost \cite{AdaC}, AdaUBoost \cite{AdaUB}, Asymmetric AdaBoost (AssymAB, \cite{AsymAB},  CSB0 \cite{CSB,EvalMetaCost}, CSB1 and CSB2 \cite{CSB}, AdaC1, AdaC2 and AdaC3 \cite{AdaC123}, and Cost-Sensitive AdaBoost (CSAB, \cite{CSAB}). Their steps different to AdaBoost are:
\begin{enumerate}
\item[1.] Assign to each record in the training set weight according to Table~\ref{table:CostBoost}.
\item[4.] Update each record's weight according to Table~\ref{table:CostBoost}.
\item[6.] A record's final outcome is the weighted majority vote on its outcome from all models, the weights being according to Table~\ref{table:CostBoost}.
\end{enumerate}
\begin{remark}
It is not immediately clear how record-dependent costs can be used in CSAB.
\end{remark}
All these have been theoretically analysed in  \cite{NeedCSAda}  together with Naive CS AdaBoost (Sect.~\ref{sssec:weighted}) and   AdaMEC (Sect.~\ref{sec:voting}) from different viewpoints, with the conclusion that only the latter two and AsymAB have solid foundations, while calibration improves performance.
\subsection{Post-Training Methods}\label{ssec:Post}
In post-training methods, misclassification costs influence the classification step. Thus, when not used to build hybrid models, they offer the advantage of not having to retrain models when costs change. Some of these methods are only applicable if the costs of the records to be predicted are known at the time of prediction, thus when this is not the case, the unknown costs need to be somehow estimated. 
\subsubsection{Direct Minimum Expected Cost Classification}\label{sssec:DMEC} Direct Minimum Expected Cost Classification (DMECC)
bases the final decision of a classifier producing probabilities on a threshold $T \in \{T_{cs},T_{thr}\} $ (see Sect.~\ref{sec:csdecisions}).
	\begin{enumerate}
	    \item Build a model producing probabilities.
		\item A record's outcome is obtained according to $P_+ > T$.
	\end{enumerate}
One possibility is to apply DMECC to an ensemble producing probabilities, as done in \textit{Calibrated AdaBoost} \cite{CSAdaProb}, where AdaBoost probabilities are obtained using Platt Scaling (see Sect.~\ref{sec:AdaB}) and $T=T_{cs}$. 

Another possibility we have identified is to use DMECC to obtain a cost-sensitive outcome (instead of the default cost-insensitive one) from the base classifiers in any ensemble, when these are capable of producing probabilities, leading us to propose \textit{DMECC-Ensemble} and \textit{DMECC-AdaBoost}, which do so respectively in Bagging and AdaBoost.
\begin{remark}
If $T_{cs}$ is constant for all records then DMECC-Ensemble and DMECC-AdaBoost should be equivalent to CS-SampleEnsemble and CS-SampleBoost respectively (excluding CPR-sampling), though relying on probabilities instead of data-sampling. This follows from the equivalence of DMECC with threshold $T_{cs}$ and CS-sampling shown in \cite{foundations} and discussed in Sect.~\ref{sec:cssampling}.
\end{remark}
\begin{remark}\label{remark:dmecc}
DMECC with threshold $T_{cs}$ is only applicable if the costs of the records to be predicted are known at the time of prediction (as $T_{cs}$ depends on them). 
\end{remark}
\subsubsection{MetaCost}
As originally proposed, MetaCost \cite{metacost}  relabels the training set using the predictions obtained by Bagging with DMECC and re-uses it to train a single classifier. We generalise this concept to use the predictions of any cost-sensitive classifier in the direction of \cite{EvalMetaCost} where AdaMEC is used (see Sect.~\ref{sec:voting} below) and CSB0.
	\begin{enumerate}
        	\item \label{mc:relabelling} Replace each training record's outcome by the one obtained from a cost-sensitive model.
		\item Build a single (cost-insensitive) model using the relabelled records. 
		\item A record's outcome is its outcome from the new model.
	\end{enumerate}
\begin{remark}MetaCost reduces cost-sensitive ensemble models to single models, which are typically more explainable but less capable of capturing all the data characteristics. As observed in \cite{EvalMetaCost}, these single models often perform worse than the ensemble models.
\end{remark}
\input{SecVoting}
\input{TableAlgList}
\subsection{Hybrid Methods}
Table~\ref{table:modeloverview} gives an overview of how post-training methods can be combined with pre- or during training methods to yield hybrid methods.
\input{TableModelsOverview}

%% file: FigAlgCat.tex
\centering
\begin{tabular}{ccc}
                                      \begin{tabular}{|c|}
                                      \hline
                                      \textbf{Pre-Training}\\
                                      \hline
                                      \hline                                        
                                      	\begin{tabular}{cc}                                                                       
                                      
                                      	\begin{tabular}{c}
                                      	\underline{\textit{Sampling based}}\\ 
                                      	CS-SampleEnsemble\\
                                      	CS-preSampleEnsemble\\
                                      	CS-SampleBoost\\
                                      	\end{tabular}

                                       	&

                                      	\begin{tabular}{c}
                                      	\underline{\textit{Weights based}}\\
                                      	Naive CS AdaBoost\\
                                      	CS-WeightedEnsemble\\                                   
                                      	\end{tabular}                                                      
                                      \end{tabular}\\                                                                                                                                                     
                                      \hline
                                      \end{tabular}                                                                                                                                                       
                                      
&
        
                                    \begin{tabular}{|c|}
                                    \noalign{\smallskip}\hline
                                    \textbf{During Training}\\
                                    \hline                                    
                                    \hline                                      
                                    CS base Ensemble\\                                                                        
                                    CS Variants of AdaBoost\\
                                    \\
                                    \hline
                                    \end{tabular}

&
                                    \begin{tabular}{|c|}
                                    \noalign{\smallskip}\hline                                    
                                    \textbf{Post-Training}\\
                                    \hline
                                    \hline                                   
                                    DMECC\\
                                    MetaCost\\
                                    CS Ensemble Voting\\
                                    \hline
                                    \end{tabular}\\
\end{tabular}

%% file: TableCostBoost.tex
\begin{sidewaystable*}
\caption{Details of the CS-variants of AdaBoost. As in AdaBoost (Sect.~\ref{sec:AdaB}), the model's total error $\epsilon$ is the sum of weights of all misclassified records, 
$h_*$ and $y_*$  are the record's outcome \textit{h} and true class \textit{y} mapped from $\{0,1\}$ to $\{-1,1\}$, and weights are normalised after initialisation and updating. $C_+$ and $C_-$ are $C_{FN}$ and $C_{FP}$ scaled within $\left(0,1\right]$.}
\label{table:CostBoost}
\centering
\setlength\tabcolsep{.7mm}
\begin{tabular}{ccccccc}
\noalign{\bigskip}
\hline\noalign{\smallskip}
Method&\begin{tabular}{c}Weight\\ Initialisation\end{tabular}&Weight Update&$\alpha$&\begin{tabular}{c}Vote's\\weight\end{tabular}&Other Details&\begin{tabular}{c}Reduction \\to AdaBoost\end{tabular}\\
\noalign{\smallskip}\hline\noalign{\smallskip}
UBoost&
\begin{tabular}{c}$w_+= C_{FN}$ \\$w_-= C_{FP}$\end{tabular}&
$w'=w\cdot e^{-\alpha y_* h_*}$&
$\frac{1}{2}\ln\left( \frac{1-\epsilon}{\epsilon}\right)$&
\begin{tabular}{c}$\alpha \left ( W_+ C_{FN} \right .$\\
$\left . - W_- C_{FP} \right)$
\end{tabular}
&
\begin{tabular}{c}
$W_+$ and $W_-$ are the total positive\\
and negative weights at the tree\\ node reached by the record\\
\end{tabular}
&$\times$\\
\noalign{\smallskip}
AdaCost&
\begin{tabular}{c}$w_+= C_{FN}$ \\$w_-= C_{FP}$\end{tabular}&
$w'=w\cdot e^{-\alpha y_* h_*\beta}$&
$\frac{1}{2}\ln\left( \frac{1-r}{1+r}\right)$&
$\alpha$
&\begin{tabular}{c}
record cost adjustment function $\beta$:\\
$\beta_{TP}=(1-C_+)/2, \  \beta_{TN}=(1-C_-)/2,$\\
$\beta_{FN}=(1+C_+)/2, \ \beta_{FP}=(1+C_-)/2,$\\
$r= \sum w \cdot y_* \cdot h_* \cdot \beta$
\end{tabular}
&$\times$
\\
\noalign{\smallskip}
AdaUBoost&
\begin{tabular}{c}$w_+= C_{FN}$ \\$w_-= C_{FP}$\end{tabular}&
$w'=w\cdot e^{-\alpha y_* h_* \left ( \frac{C_{FN}}{C_{FP}} \right )^{y}}$&
$\frac{1}{2}\ln\left( \frac{1-\epsilon}{\epsilon}\right)$&
$\alpha$
&
&if $C_{FN}=C_{FP}$
\\\noalign{\smallskip}
AsymAB&
\begin{tabular}{l}
$w_+=\sqrt[2m]{ \frac{C_{FN}} { C_{FP} } }\cdot \frac{1}{|N|}$\\
$w_-=\sqrt[2m]{\frac{C_{FP}}{C_{FN}}} \cdot \frac{1}{|N|}$
\end{tabular}&
\begin{tabular}{l}
$w_+'=\sqrt[2m]{ \frac{C_{FN}} { C_{FP} } } \cdot w_+\cdot e^{-\alpha h_*}$\\
$w_-'=\sqrt[2m]{\frac{C_{FP}}{C_{FN}}} \cdot w_-\cdot e^{\alpha h_*}$
\end{tabular}&
$\frac{1}{2}\ln\left( \frac{1-\epsilon}{\epsilon}\right)$&
$\alpha$
&\begin{tabular}{c}$m$: $\#$times steps~\ref{item:normalise} to \ref{item:abprobrep}\\ are repeated in step~\ref{item:abrepeat}
\end{tabular}
&if $C_{FN}=C_{FP}$
\\\noalign{\smallskip}
CSB\textit{j}&
\begin{tabular}{c}$w_+= C_{FN}$ \\$w_-= C_{FP}$\end{tabular}&
\begin{tabular}{l}
TP  $\&$ TN:  $w'=w\cdot e^{-\alpha_j}$,\\
FN: $w_+'=C_{FN} \cdot w_+\cdot e^{\alpha_j}$,\\
FP:  $w_-'=C_{FP}\cdot w_-\cdot e^{\alpha_j}$ 
\end{tabular} &
$\frac{1}{2}\ln\left( \frac{1-\epsilon}{\epsilon}\right)$ 
&  
$\alpha \cdot C_h$ 
&
\begin{tabular}{c}
$\alpha_0=0, \alpha_1=1, \alpha_2=\alpha$,\\
$C_h \in \left\{ C_{FN} , C_{FP} \right\}$ is the misclassification\\
cost associated with the record's outcome \textit{h}
\end{tabular}
&\begin{tabular}{c} if $j=2$ and \\ $C_{FN}=1=C_{FP}$ \end{tabular}
\\\noalign{\smallskip}
AdaC1&$w=1$&
\begin{tabular}{l}$w_+'=w_+\cdot e^{-\alpha h_*C_{+}}$\\
                  $w_-'=w_-\cdot e^{\alpha h_*C_{-}}$
\end{tabular}&
$\frac{1}{2}\ln\left( \frac{1+r_t-r_f}{1-r_t+r_f}\right)$&
$\alpha$
&$r_t=\displaystyle \sum_{y=h} w_+ \cdot C_{+} + \displaystyle \sum_{y=h} w_- \cdot  C_{-}\enspace,$
&if $C_{FN}=1=C_{FP}$\\ \noalign{\smallskip}
AdaC2&$w=1$&
\begin{tabular}{l}
$w_+'=C_{FN} \cdot w_+\cdot e^{-\alpha h_*}$\\
$w_-'=C_{FP}\cdot w_-\cdot e^{\alpha h_*}$
\end{tabular}&
$\frac{1}{2}\ln\left( \frac{r_t}{r_f}\right)$&
$\alpha$
&$r_f=\displaystyle \sum_{y\ne h} w_+ \cdot C_{+} + \displaystyle \sum_{y\ne h} w_- \cdot C_{-}\enspace,$
&if $C_{FN}=C_{FP}$\\ \noalign{\smallskip}
AdaC3&$w=1$&
\begin{tabular}{l}$w_+'=C_{+} \cdot w_+\cdot e^{-\alpha h_*C_{+}}$\\
$w_-'=C_{-}\cdot w_-\cdot e^{\alpha h_*C_{-}}$\end{tabular}&
$\frac{1}{2}\ln\left( \frac{r_t+r_f+r_{2t}-r_{2f}}{r_t+r_f-r_{2t}+r_{2f}}\right)$&
$\alpha$
&\begin{tabular}{l}
$r_{2t}=\displaystyle \sum_{y=h} w_+ \cdot C_{+}^2 + \displaystyle \sum_{y=h} w_- \cdot C^2_{-}\enspace,$\\
$r_{2f}=\displaystyle \sum_{y\ne h} w_+ \cdot C^2_{+} + \displaystyle \sum_{y\ne h} w_- \cdot C^2_{-}$ \end{tabular}
&if $C_{FN}=1=C_{FP}$ \\ \noalign{\smallskip}
CSAB&
\begin{tabular}{c}$w_+=\frac{1}{|N^+|}$\\$w_-=\frac{1}{|N^-|}$\end{tabular}&
\begin{tabular}{l}
$w_+'=w_+\cdot e^{-\alpha h_* C_{FN}}$\\
$w_-'=w_-\cdot e^{\alpha h_*C_{FP}}$ 
\end{tabular}&
\begin{tabular}{c}
$2b\cdot C_{FN} \cdot \cosh(C_{FN}\alpha)+$\\
$2d \cdot C_{FP}  \cdot \cosh(C_{FP}\alpha)$\\ 
$=C_{FN}\cdot e^{-C_{FN}\alpha} \sum w_+ +$\\ 
$C_{FP}\cdot e^{-C_{FP}\alpha} \sum w_-$
\end{tabular}&
$\alpha$
&\begin{tabular}{c}$b=\displaystyle \sum_{h=0} w_+$, $d=\displaystyle \sum_{h=1} w_-$\end{tabular}
&\begin{tabular}{c}if $C_{FN}=1=C_{FP}$\\ and $|N^+|=|N^-|$\end{tabular}\\
\noalign{\smallskip}\hline
\end{tabular}
\end{sidewaystable*} 

%% file: SecVoting.tex
\subsubsection{CS Ensemble Voting}\label{sec:voting}
Costs can also be taken into account in an ensemble during weighted majority voting.
\paragraph*{\textbf{Cost-Sensitive Weights for Model Votes}} 
In certain AdaBoost variants, such as Naive CS AdaBoost, $\epsilon$ is calculated on cost-based record weights and hence results in a cost-sensitive $\alpha$, which serves as the weight of the model's vote in weighted majority voting.  
We observe that it is in fact straightforward to mimic this for any ensemble as follows:
\begin{enumerate}
\item Assign to each positive and negative record in the training set weights as in (\ref{eq:posnegweight}).
\item For each model in the ensemble compute $\alpha=f(\epsilon)$, where $f$ is some function and $\epsilon$ is the sum of weights of all misclassified records from the training or a validation set.
\end{enumerate}
Possibilities for the function $f$ include
\begin{equation}\label{eq:csweigts}
f(\epsilon)=\ln\left( (1-\epsilon)/\epsilon\right), \ \  f(\epsilon)=1-\epsilon, \ \  f(\epsilon)=e^{(1-\epsilon)/\epsilon} \text{ and } f(\epsilon)=\left((1-\epsilon)/\epsilon\right)^2 \enspace,
\end{equation}
 the latter two providing a right-skewed distribution.
\paragraph*{\textbf{MEC-Voting}} By MEC-Voting we shall refer to the generalisation to 
any ensemble of  
the idea behind \textit{AdaBoost with minimum expected cost criterion} \cite{CSB,EvalMetaCost}, or \textit{AdaMEC} as coined
 in \cite{CSAdaProb}, which is AdaBoost with modified Step~\ref{item:adaboostfinal}:
\begin{enumerate}
\item[6.] A record's final outcome is the weighted majority vote on its outcome from all models,  the weights being the product of $\alpha$ and the misclassification cost associated with the outcome.
\end{enumerate}  
\paragraph*{\textbf{Majority Threshold Adjustment (MTA)}} The outcome of weighted majority voting is positive if the sum of positive votes is greater than 0.5. Alternative cost-sensitive \textit{majority thresholds} that can be used are $T_{cs}$ (Sect.~\ref{sec:csdecisions}) and $T_{mthr}$ which we define as the one that yields the least costs on a validation set as done in Thresholding described in Sect.~\ref{sec:csdecisions}.
\begin{theorem}\label{thm:mec-mtaEquiv}
MEC-Voting and MTA using $T_{cs}$ are equivalent.
\end{theorem}
\begin{proof}
The weighted sums of positive votes with and without  MEC-Voting in an ensemble of $m$ models are respectively $S_1= \frac{\sum_{i:h_i=1}{\alpha_i C_{h_i}}}{ \sum_{i=1}^m{\alpha_i C_{h_i}}}$ and $S_2=\frac{\sum_{i:h_i=1}{\alpha_i}}{ \sum_{i=1}^m{\alpha_i }}$, where 
$C_{h_i}\in \{C_{FN},C_{FP}\}$ is the (non-zero) misclassification cost associated with the record's outcome $h_i$ from model $i$. 

If $S_1=0$ then $S_2=0$ and the theorem holds trivially. Otherwise, $S_1$ can be expressed in terms of $S_2$:
$$\begin{array}{*{20}{l}}
S_1&=&\frac{\sum_{i:h_i=1}{\alpha_i C_{h_i}}}{ \sum_{i:h_i=1}{\alpha_i C_{h_i}}+\sum_{i:h_i=0}{\alpha_i C_{h_i}}}
   &=&\frac{1}{1+\frac{\sum_{i:h_i=0}{\alpha_i C_{h_i}}}{\sum_{i:h_i=1}{\alpha_i C_{h_i}}}}
   &=&\frac{1}{1+\frac{C_{FP}}{C_{FN}}\left( \frac{\sum_{i:h_i=0}{\alpha_i}}{\sum_{i:h_i=1}{\alpha_i}} + 1 - 1 \right) }\\
   &=&\frac{1}{1+\frac{C_{FP}}{C_{FN}}\left( \frac{\sum_{i=1}^m{\alpha_i}}{\sum_{i:h_i=1}{\alpha_i}} - 1 \right) }
   &=&\frac{1}{1+\frac{C_{FP}}{C_{FN}}\left( \frac{1}{S_2} - 1 \right) } \enspace.\\   
\end{array}$$
Solving this equality for $S_2$ and using the fact that a record's outcome is positive if $S_1>0.5$ and negative otherwise, we obtain $S_2 > \frac{C_{FN}}{C_{FN}+C_{FP}}$, which is MTA using $T_{cs}$ as required.\end{proof}
\begin{remark} 
The equivalence of Theorem~\ref{thm:mec-mtaEquiv} was shown specifically for AdaMEC in \cite{NeedCSAda}.
\end{remark}
\begin{remark}\label{remark:mec&MT}
Both MEC-Voting and MTA using $T_{cs}$ are only applicable if the costs of the records to be predicted are known at the time of prediction.
\end{remark}

%% file: TableAlgList.tex
\begin{table}[!t]
\caption{The basic ensembles and a (non-exhaustive) list of pre- and during-training methods derived from our framework (independent of base classifier), with abbreviations. Novel ones 
 identified in this work are indicated by an asterisk *. The DMECC-Ensemble prefix \textit{dm-} is applicable to all non-AdaBoost ensembles in the list.}
\label{table:algoverview}
\centering
\begin{tabular}{cclccl}
\noalign{\bigskip}
\hline\noalign{\smallskip}
$\#$&Abbrv.&Name&$\#$&Abbrv.&Name\\
\noalign{\smallskip}\hline\noalign{\smallskip}
1.&bg&Bagging (Bg)          			&19.&upbg&Under-preSampleEnsemble - Bg *\\
2.&rf&Random Forests (RF)            &20.&cprpbg&CPR-preSampleEnsemble - Bg *\\       
3.&rdf&Random Decision Forests (RDF) &21.&opbg&Over-preSampleEnsemble - Bg *\\
4.&wbg&weightedEnsemble - Bg  *      &22.&uprf&Under-preSampleEnsemble - RF *\\
5.&wrf&weightedEnsemble - RF         &23.&cprprf&CPR-preSampleEnsemble - RF *\\
6.&wrdf&weightedEnsemble - RDF  *    &24.&oprf&Over-preSampleEnsemble - RF *\\
7.&ab&AdaBoost - AB    				&25.&uprdf&Under-preSampleEnsemble - RDF *\\
8.&ncsab&Naive CS AB         			&26.&cprprdf&CPR-preSampleEnsemble - RDF *\\
9.&ac\textit{i}&AdaC\textit{i}, $i \in \{1,2,3\}$     &27.&oprdf&Over-preSampleEnsemble - RDF *\\ 
10.&acost&AdaCost    				&28.&ubg&Under-SampleEnsemble - Bg \\
11.&aub&AdaUBoost                   	&29.&cprbg&CPR-SampleEnsemble - Bg\\
12.&csa&CSAB                      	&30.&obg&Over-SampleEnsemble - Bg\\
13.&csb\textit{i}&CSB\textit{i}, $i \in \{0,1,2\}$    &31.&urf&Under-SampleEnsemble - RF *\\     
14.&asb&Asymmetric AB                   				&32.&cprrf&CPR-SampleEnsemble - RF *\\
15.&usb&Under-SampleBoost                  		    &33.&orf&Over-SampleEnsemble - RF *\\
16.&cprb&CPR-SampleBoost *                       	    &34.&urdf&Under-SampleEnsemble - RDF *\\
17.&dab  &DMECC AB*                         			&35.&cprrdf&CPR-SampleEnsemble - RDF *\\
18.&dm-&DMECC-Ensemble*          						&36.&ordf&Over-SampleEnsemble - RDF *\\                                                                   
\noalign{\smallskip}\hline
\end{tabular}
\end{table}

%% file: TableModelsOverview.tex
\begin{table}[!t]        
\caption{Overview of the post-training methods combinable with each type of pre- and during-training method 
 according to the probability calibration used. There are four main components, namely DMECC, type of the ensemble's output, MTA and MetaCost, and their combination yields different hybrid models. Considering all combinations is a novelty of this work, and so are the two components indicated by an asterisk *.   
Avg(pr) denotes the average probability of the models in the ensemble, and wtMaj(cls) the class predicted by their weighted majority, considering both equal and cost-sensitive weights $\alpha=f(\epsilon)$, the latter given by all functions \textit{f} specified in (\ref{eq:csweigts}).} 
\label{table:modeloverview} 
\centering
\begin{tabular}{*{8}{c}} 
\noalign{\bigskip}
\hline\noalign{\smallskip}
Method                   &Probability                                  &\multicolumn{2}{c}{Ensemble's Output}& \multicolumn{2}{c}{DMECC}         &MTA*              &Meta\\ 
Type                     &Calibration            &Avg(pr)    &wtMaj(cls)*   & $T_{cs}$ &  $T_{thr}$                   &
																																							   &Cost\\
\noalign{\smallskip}\hline\noalign{\smallskip}
                         &None               &$\times$    &\checkmark    &$\times$   &$\times$    &\checkmark &\checkmark\\
AdaBoost \&                    &Logistic Correction&$\times$    &\checkmark    &\checkmark        &\checkmark  &\checkmark &\checkmark\\
variants                 &Platt Scaling      &$\times$    &\checkmark    &\checkmark        &\checkmark   &\checkmark &\checkmark\\
 &Isotonic Regression     &$\times$    &\checkmark    &\checkmark        &\checkmark   &\checkmark &\checkmark\\
\vspace{2mm}\\

\begin{tabular}{c}
CS-Weighted\\ensemble
\end{tabular}            &None               &\checkmark     &\checkmark &\checkmark    &\checkmark  &\checkmark &\checkmark
\vspace{2mm}\\

                         &None               &\checkmark     &\checkmark &\checkmark   &\checkmark  &\checkmark &\checkmark\\
Others                   &Laplace Smoothing &\checkmark     &\checkmark &$\times$   &\checkmark &\checkmark &\checkmark\\
                         &m-estimation       &\checkmark     &\checkmark &$\times$   &\checkmark  &\checkmark &\checkmark\\
                         &curtailment      &\checkmark     &\checkmark &$\times$   &\checkmark  &\checkmark &\checkmark\\
\noalign{\smallskip}\hline
\end{tabular}
\end{table}

%% file: SecApplic.tex
\section{Towards Determining the State-of-the-Art in Cost-Sensitive Learning}\label{sec:experiments}
A natural question to ask is which, if any, of the described framework components can be considered as state-of-the-art.  Obtaining an indication on this requires a rigorous experimental comparison over a range of datasets, often referred to as \textit{benchmarking}. Such a benchmarking would be most useful if it considers sufficiently many publicly available datasets in order to allow reproducibility and the updating of the benchmarking via the inclusion of newly proposed methods. As already mentioned in the Introduction (Sect.~\ref{sec:RelatedWork}), there are two main uses of cost-sensitive learning, which should be considered separately.

 The first use is for treating class-imbalance alone,  
 in which case the misclassification costs do not necessarily have to be derived from the dataset or application domain and can be randomly assigned. Typically, different pairs of constant (class-dependent) costs are tried out in search for the one that gives the best results according to the metric of choice, which should be suitable for imbalanced datasets. A benchmarking can therefore be performed on a selection of the many imbalanced datasets already publicly available, and different sub-cases can depend on the level of imbalance. Although such a benchmarking will give an indication on the framework components that are best suited for treating class imbalance, in order to provide a complete picture it needs to part of a more general benchmarking that includes cost-insensitive methods specifically aiming at treating imbalance (see for example \cite{surGalar} for an overview) as well. 
 
 The second use is for minimising misclassification costs that are derived from the application domain or dataset, irrespective of the level of imbalance. Typically, the evaluation measure depends on these costs, often simply being their sum. This, however, might not be sufficient in certain cases, which include fraud detection and direct marketing, and an additional evaluation measure might be necessary. 
For instance, in fraud detection we are in practice interested in models that do not disrupt the operation of the business, thus our attention should be restricted to models that not only achieve the lowest costs, but also a realistic False Positive Rate (FPR). In the case of credit cards in particular, a FPR of at most $3\%$ should be within the capacity of investigating agents. 
In direct marketing scenarios, contacting potential responders to a request (such as for making a donation or a purchase) may incur costs (such as for postage). Thus, the application of a model assumes that a budget that provides the capacity to contact all those the model predicts as responders is readily available, which might not always be the case. For this reason it would be more appropriate to also look at the \textit{return on investment} (ROI), given as the net profit over expenditure. These two examples suggest that a benchmarking should probably be done per application domain, something not unusual (see for instance  \cite{lessbench} for a benchmark in the domain of credit scoring, albeit not focusing on cost-sensitive methods).

The main obstacle preventing such a benchmarking is the absence of sufficiently many publicly available datasets in general, let alone per application domain. The alternative of using publicly available datasets having no attribute from which misclassification costs can be derived, and assigning random costs to them (as done for just treating class imbalance we discussed above), might give misleading results in the absence of a clear business case. Moreover, constant costs are quite rare in domains such as direct marketing and fraud detection, and assigning random record-dependent costs is a difficult task, mainly because the distribution from they should be taken is unknown. 
We hope that practitioners will receive this as an open call to make more datasets publicly available in order to facilitate advances in the field, from which they can benefit as well. 

For the interested reader, in Appendix~\ref{appendix}, we provide three examples of cost-minimisation applications  of cost-sensitive learning using publicly available datasets.

%% file: SecConc.tex
\section{Conclusions}\label{sec:conclusions}
In this paper we have described and categorised available cost-sensitive methods with respect to misclassification costs by means of a unifying framework that also allowed us to identify new combinations and extend ideas across classifiers. This was our main contribution, which clarifies the picture and should aid further developments in the domain and serve as a baseline to which newly proposed methods should be compared.
\subsection{Future Work} \label{sec:FutureWork}
As our work has identified, some possibilities for new cost-sensitive ensemble methods can arise by developing a novel approach in any of the following domains: (a) cost-sensitive base classifiers such as decision trees with cost-sensitive node splitting and pruning, (b) cost-sensitive sampling, and using costs to (c) specify record weights, (d) update weights in AdaBoost variants, and (e) specify classifier weights for ensemble voting.

Worth exploring are how \textit{Ensemble Pruning} \cite{EnsPrun} and \textit{Stacking} \cite{stacking} (which instead of using the outputs of all the ensemble members for voting or averaging, respectively first choose a subset of them, or use them to train a second model) can be made cost-sensitive for inclusion in the post-training methods. It would also be interesting to investigate whether \textit{Gradient Boosting} \cite{gbm}, another representative of the boosting principle, and its popular variant \textit{XGBoost} \cite{xgboost} can have cost-sensitive variants, in particular with record-dependent costs, apart from being combined with post-training methods.

Another avenue for future research is to examine cost-sensitivity with respect to other types of costs as mentioned in \cite{CostTypes}, particularly costs of attribute acquisition and costs of obtaining missing values in the data \cite{ICET,DTMC}, that are important in many real world applications.

%% file: Appendix.tex
\appendix
\section{Example Applications}\label{appendix}
Here we briefly describe three examples of how cost-sensitive learning can be applied in real world scenarios where the main task is cost-minimisation. They are all based on publicly available datasets.	
\subsection{Datasets}\label{ssec:datasets}
\textit{Direct Marketing.}  The first dataset is on a donations to charity campaign in the United States and is available on-line as part of the 1998 KDD cup\footnote{\url{https://kdd.ics.uci.edu/databases/kddcup98/kddcup98.html}}.
The dataset contains 483 attributes related to people contacted by mail and asked to donate to a charity, including the donation amount (where 0 indicates no donation). It is a highly imbalanced set whose positive records are 5.1\% of it.

\textit{Churn Prediction.} The second dataset is on churn prediction in the Telecom industry and is available on-line at Kaggle\footnote{\url{https://www.kaggle.com/blastchar/telco-customer-churn}}. 
It comprises 7043 customer records of 21 attributes, including whether they churned or not. 
It is a moderately imbalanced set whose positive records  constitute 26.54\% of it.

\textit{Credit Card Fraud Detection.}  
 The third and final dataset is on credit card fraud detection and is available on-line at Kaggle\footnote{\url{https://www.kaggle.com/mlg-ulb/creditcardfraud/}}. 
In total, two days of credit card transactions are given  (284807 records of 31 attributes). For obvious security and user privacy reasons, all the attributes names and values have been scrambled and transformed into numerical ones by the owners, apart from the transaction's amount, the time (in seconds) elapsed between the transaction and the first transaction in the dataset, and whether the transaction was fraudulent or not. It is an extremely imbalanced set whose positive records are a mere 0.17\% of it.
\input{TableDataSets} 
\subsection{Data Preparation}\label{ssec:dataprep} 
The datasets were segmented into three disjoint subsets respecting the global proportion of positive records: a training set used for building models, a validation set used for determining parameters such as $T_{thr}$ and $\epsilon$ used for specifying model weights, and a test set used for model evaluation. Table~\ref{table:dataoverview} provides the details.

\textit{Direct Marketing:} 
The data is provided in two parts. Following the competition guidelines, we used the first part for training and validation and the second one for evaluation. As the amount of attributes the dataset contains is relatively large, we followed the selection and engineering approach of \cite{KDD98DataPrep}, the competition winners,  to reduce their number significantly in order speed up model training. 

\textit{Churn Prediction:} 
There was no restriction on the splitting. The redundant customer ID attribute was dropped.  
Although the dataset does not have predefined costs, it does have an attribute from which they can be derived, namely \textit{Monthly Costs}. Assuming a (realistic) campaign where customers contacted are offered 2 months for free if they renew their  telephone subscription, $C^i_{FP}$ can be set to two times \textit{amt}, the amount charged monthly, and $C^i_{FN}$ to twelve times \textit{amt} (equivalent to a whole year of lost profits from  the churning customers).

\textit{Credit Card Fraud Detection:} 
Credit card fraud models are typically built on past transactions in order to detect fraud amongst new transactions. With this in mind, we used the first day for training and validation and the other for evaluation. All attributes were used. 
\subsection{Misclassification Costs}
Models were trained using both the class and record-dependent dataset-specific cost pairs $(C_{FP},C_{FN})$ that are mentioned below.  
However, $T_{cs}$, $T_{thr}$ and $\epsilon$ used for  model weights were always computed using the record-dependent costs.
 
\textit{Direct Marketing:} (0.68, 14.45) and (0.68, $amt-0.68$), where $0.68$ is the cost of contacting a person, $14.45$  is the average donation amount of the donors in the training dataset minus $0.68$, and \textit{amt} is the actual record-dependent donation's amount for positive records. To be able to be able to use MEC-Voting and MTA and DMECC with $T_{cs}$ (see Remarks~\ref{remark:dmecc} and~\ref{remark:mec&MT}, we estimated \textit{amt} for negative records as the would-have-been donation if the record was in fact positive via a linear regression, again following \cite{KDD98DataPrep}. 
 
\textit{Churn Prediction:} (1, 6), where 1 to 6 is the ratio of $C_{FP}^i$ over $C_{FN}^i$ as derived from the campaign we set up in Sect.~\ref{ssec:dataprep}. 

\textit{Credit Card Fraud Detection:} (1, \textit{amt}), (1, 119.68), (2, \textit{amt}), (2, 119.68), (5, \textit{amt}) and (5, 119.68), where  $119.68$  is the average amount of fraudulent transactions in the training data, \textit{amt} is the actual record-dependent transaction's amount, and 1, 2 and 5 are the (class-dependent) overhead costs we consider in the absence of a clear idea of what they are.
\subsection{Evaluation}\label{sec:eval} We evaluate models based on their cost saving performance using the record-dependent misclassification costs we defined per dataset. For reference, we also include \textit{True} and \textit{False positive rates} (TPR and FPR), and \textit{area under the ROC curve} (AUC \cite{provost2001robust}). These are standard accuracy-based measures, and while the former two depend on the decision threshold used, the latter evaluates performance across all possible thresholds.

\textit{Direct Marketing:} The total net profit is computed  as the sum of donations by people contacted minus a cost of 0.68 per person contacted. As discussed in Sect.~\ref{sec:experiments}, we also look at ROI. 
For example, the trivial model of targeting everyone on the list yields a net profit of $10560.08$, but only if the budget to cover the associated $65529.56$ in postage costs is available, leading to a ROI of $0.16$. Next, consider a naive CS AdaBoost model with net profit $9282.59$ that has ROI $0.59$. Even though the profit is about $1300$ less, the ROI is almost $4$ times larger, meaning that only one fourth of the trivial model's budget is needed.

\textit{Churn Prediction:} We calculated Cost\%, the percentage by which costs are reduced due to the campaign as compared to having no campaign and losing all churners, with the assumption that all contacted churners accept the offered 2 months of free subscription and are retained.

\textit{Credit Card Fraud Detection:} We used \textit{amt} as $C^i_{FN}$ and computed the percentage of the sum of the amounts of all fraudulent transactions that is saved by a model by correctly detecting fraud, which we denote by TotF$\%$. As discussed in Sect.~\ref{sec:experiments}, we restrict our attention those that achieve a realistic FPR of at most $3\%$.
\subsection{Experiments}
We implemented our framework in the software package R \cite{R}, using a combination of existing and our own implementations. All ensembles we built consist of 100 rpart trees (the implementation of CART in R),  
 for Random Forests we allowed the growing of deep trees, whereas decision stumps (trees of depth 2) were grown for AdaBoost and variants.
All experiments were run three times using a different splitting of the dataset into training and validation sets (see Sect.~\ref{ssec:dataprep}), and the results were averaged. 
 Tables~\ref{table:KDDoverview}, \ref{table:Churnoverview} and \ref{table:FDoverview} show the best-performing models obtained from our experiments.
\input{KDDROITable.tex}
\input{ChurnTable.tex}
\input{FDTable.tex}

%% file: TableDataSets.tex
\begin{table}[!t]
\caption{Overview of the 3 datasets.
 For the first one, within brackets is the number obtained after attribute engineering.}
\label{table:dataoverview}           
\centering
\begin{small}
\begin{tabular}{*{9}{c}}
\noalign{\bigskip}
\hline\noalign{\smallskip}
$\#$&Domain&Set&\multicolumn{2}{c}{Size}&+ves&\multicolumn{2}{c}{$\#$Attributes}&Training \\
&&&$\#$&$\%$&($\%$)&Total &Used &($C_{FP}, C_{FN}$)\\
\noalign{\smallskip}\hline\noalign{\smallskip}
1.&Direct Marketing        & Total     &191779&100  & 5.1 &483&31&(0.68, 14.45),\\
  &     &Training   &65000 &33.9 & 5.1&&(8)  &(0.68, $amt-0.68$)\\
  &              &Validation &30412 &15.9 &5  &    &    &\\
  &              &Test       &96367 &50.2 &5   &    &    & 
\vspace{3mm}\\

2.&Churn  Prediction     & Total     &7043&100& 26.54 &21&  20 &(1, 6), \\
  &  &Training   &4508&64& 26.77  & &     & (122.35, 893.09)\\
  &            &Validation &1127&16&26.44   &    &       &\\
  &            &Test       &1408&20&25.85   &    &       &
  \vspace{3mm}\\

3.&Credit Card   &Total     &284807&100&0.17 &31&31 &(1, \textit{amt}), (1, 119.68),\\
  & Fraud Detection       &Training  &110852&38.9& 0.18&&  & (2, \textit{amt}), (2, 119.68),\\
  &            &Validation &33934&11.9&0.18  &    &    &(5, \textit{amt}), (5, 119.68)\\
  &              &Test       &140021&49.2&0.15 &    &    & \\
  
\noalign{\smallskip}\hline
\end{tabular}
\end{small}
\end{table}

%% file: KDDROITable.tex
\begin{table*}[!t]
\caption{The 10 models with the highest profit in the Direct Marketing Dataset, using the abbreviations from Tables~\ref{table:algoverview} and \ref{table:modeloverview}. We observe that their ranking according to the ROI column would be different. All DMECC-ensembles happen to use $T_{cs}$, and  the best performing cost-insensitive model (*) is included for reference.}
\label{table:KDDoverview}
\centering
\begin{scriptsize}
\begin{tabular}{*{20}{c}} 
\noalign{\bigskip}
\hline\noalign{\smallskip} 
\#&Method&Training&Probability&Ensemble's&DMECC&$\alpha$&Profit&ROI&TPR(\%)& FPR(\%)&AUC \\
   &    & $(C_{FP},C_{FN})$             &Calibration&Output    &/ MTA   &    &       &     &       &        &\\
\noalign{\smallskip}\hline\noalign{\smallskip}
1.&dm-rdf&          &m-estim.    &wtMaj(cls)&         &$\left(\frac{1-\epsilon}{\epsilon}\right)^2$
                                                                      &  \textbf{13769.75}  & 0.28 &69.37  &  75.40  & 0.458\\                                                                                                      
2.&dm-rf&           &Laplace&wtMaj(cls)&&$e^{\frac{1-\epsilon}{\epsilon}}$   
                                                                      &  13729.00  &  0.37 &59.97  &  56.93    & 0.519\\ 
3.&dm-rf &         &&wtMaj(cls)&&$e^{\frac{1-\epsilon}{\epsilon}}$  
                                                                      &  13700.18  & 0.37 &59.03  &  55.97    & 0.519\\                                                                                                                                            
4.&dm-rf  &         &&wtMaj(cls)&&$\left(\frac{1-\epsilon}{\epsilon}\right)^2$  
                                                                      &  13698.68  & \textbf{0.38} &58.80  &  \textbf{55.67}  & 0.520\\ 
5.&dm-rf &         &Laplace&wtMaj(cls)&&$\left(\frac{1-\epsilon}{\epsilon}\right)^2$    
                                                                      &  13683.40  & 0.37 &59.73  &  56.63   & 0.519\\ 
6.&dm-rf  &         &Laplace&wtMaj(cls)&&$1-\epsilon$
                                                                      &  13679.26  & 0.37 &60.13  &  57.20  & 0.519\\ 
7.&rdf &      &   &Avg(pr)&$T_{cs}$& 
                                                                      &  13674.27  & 0.28 &69.30  &  75.03   & \textbf{0.603}\\ 
8.&dm-rdf &        &m-estim.&wtMaj(cls)&&$e^{\frac{1-\epsilon}{\epsilon}}$
                                                                      &  13662.86  & 0.28 &\textbf{69.70}  &  76.03   & 0.458\\ 
9.&dm-rf  &         &&wtMaj(cls)&&$1-\epsilon$
                                                                      &  13662.41  & 0.37 &59.17  &  56.20   & 0.519\\                                                                       
10.&dm-rdf&&Laplace& wtMaj(cls)&&$\left(\frac{1-\epsilon}{\epsilon}\right)^2$& 13654.57& 0.30&63.40& 69.77          &0.458\\
  \noalign{\smallskip}\hline\noalign{\smallskip}                                                                    
*&rpart&&m-estim.&&&&                        94.60&    &0.20    &0.10        &0.594\\                                                                      
\noalign{\smallskip}\hline
\end{tabular}
\end{scriptsize}
\end{table*}

%% file: ChurnTable.tex
\begin{table*}[!t]
\caption{The 10 models with the highest Cost\% in the Churn Prediction Dataset, using the abbreviations from Tables~\ref{table:algoverview} and \ref{table:modeloverview}. 
 All DMECC-ensembles happen to use $T_{cs}$. The best performing cost-insensitive model (*) is included for reference.}
\label{table:Churnoverview}
\centering
\begin{scriptsize}
\begin{tabular}{*{20}{c}} 
\noalign{\bigskip} 
\hline\noalign{\smallskip} 
\#&Method&Training&Probability&Ensemble's&DMECC&$\alpha$&Cost\%&TPR(\%)& FPR(\%)&AUC \\
   &    & $(C_{FP},C_{FN})$             &Calibration&Output    &/ MTA      &       &     &       &        &\\
\noalign{\smallskip}\hline\noalign{\smallskip}
1.&acost&$(1,6)$&Platt&wtMaj(cls))&$T_{cs}$&$\log(\frac{1-\epsilon}{\epsilon})$				 	&\textbf{70.84} &90.70 &44.00  &0.834\\
2.&asb&$(1,6)$&&wtMaj(cls)&&$\log(\frac{1-\epsilon}{\epsilon})$                     				 	&70.84 &90.40 &44.40  &\textbf{0.839}\\
3.&asb&$(1,6)$&Logistic&wtMaj(cls)&&$\log(\frac{1-\epsilon}{\epsilon})$            				 	&70.84 &90.40 &44.40  &0.839\\
4.&ab&&Logistic&wtMaj(cls)&$T_{cs}$&$\log(\frac{1-\epsilon}{\epsilon})$         					&70.77 &90.10 &44.20  &0.838\\

5.&dm-bg&&&wtMaj(cls)&&$\left( \frac{1-\epsilon}{\epsilon} \right) ^2$               &70.61& 91.80& 47.63&      0.811\\
6.&dm-bg&&Laplace&wtMaj(cls)&&$\left( \frac{1-\epsilon}{\epsilon} \right) ^2$          &70.58& \textbf{92.00}& \textbf{48.20}&      0.811\\
7.&dm-bg&&&wtMaj(cls)&&$e^{\frac{1-\epsilon}{\epsilon}}$              &70.57& \textbf{92.00}& \textbf{48.20}&       0.812\\

8.&asb&$(1,6)$&Platt&wtMaj(cls)&$T_{cs}$&$\log(\frac{1-\epsilon}{\epsilon})$ 				 	&70.57 &90.10 &44.40  &0.839\\
9.&ab&&Platt&wtMaj(cls)&$T_{cs}$&$\log(\frac{1-\epsilon}{\epsilon})$         				 	&70.55 &90.10 &44.60  &0.838\\

10.&dm-bg&&Laplace&wtMaj(cls)&&$1$&                     70.53& 91.80& 47.87&     0.810\\
\noalign{\smallskip}\hline\noalign{\smallskip}
*&ab&&Platt&wtMaj(cls)&&$\log(\frac{1-\epsilon}{\epsilon})$&51.04&53.00&10.60&0.838\\
\noalign{\smallskip}\hline
\end{tabular}
\end{scriptsize}
\end{table*}

%% file: FDTable.tex
\begin{table*}[!t]
\caption{The 5 models with the highest TotF\% and FPR respectively at most $1\%$, $2\%$ and $3\%$ in the Credit Card Fraud Detection Dataset, using the abbreviations from Tables~\ref{table:algoverview} and \ref{table:modeloverview}. 
 All DMECC-ensembles use $T_{cs}$ except one that uses $T_{thr}$. The best performing cost-insensitive model (*) is included for reference.}
\label{table:FDoverview}
\centering
\begin{scriptsize}
\begin{tabular}{*{20}{c}}
\noalign{\bigskip} 
\hline\noalign{\smallskip} 
\#&Method&Training&Probability&Ensemble's&DMECC&$\alpha$&TotF\%&FPR(\%)& TPR(\%)&AUC \\
   &    & $(C_{FP},C_{FN})$             &Calibration&Output    &/ MTA      &       &     &       &        &\\
\noalign{\smallskip}\hline\noalign{\smallskip}
1.&dm-upbg&$(5,119.31)$&&wtMaj(cls)&&$\log(\frac{1-\epsilon}{\epsilon})$  &                           \textbf{80.67}& 0.97& 75.83&   0.889\\
2.&dm-oprdf&$(5,119.31)$&m-estim.&wtMaj(cls)&&$\left( \frac{1-\epsilon}{\epsilon} \right) ^2$&       80.42& 0.87& 70.93&   0.891\\
3.&dm-oprdf&$(1,119.31)$&Laplace&wtMaj(cls)&&$\log(\frac{1-\epsilon}{\epsilon})$&                    79.86& \textbf{0.83}& 74.23&   0.889\\
4.&dm-cprbg&$(5,119.31)$&&wtMaj(cls)&&$\left( \frac{1-\epsilon}{\epsilon} \right) ^2$&                79.52& 0.97& 75.67&   0.884\\
5.&dm-cprbg&$(5,amt)$&m-estim.&wtMaj(cls)&$T_{cs}$&$\left( \frac{1-\epsilon}{\epsilon} \right) ^2$& 77.26& 0.87& \textbf{76.13}&   \textbf{0.959}\\ 
\noalign{\smallskip}\hline\noalign{\smallskip}                                                                             
1.&dm-uprdf&$(5,119.31)$&Laplace&wtMaj(cls)&&$1$&                                \textbf{88.69}& 1.83& 68.23&      \textbf{0.891}\\
2.&dm-cprrdf&$(5,119.31)$&Laplace&wtMaj(cls)&&$1$&                               88.20& 1.73& 68.40&      0.877\\
3.&dm-cprrdf&$(5,119.31)$&Laplace&wtMaj(cls)&&$\log(\frac{1-\epsilon}{\epsilon})$&88.15& 1.63& \textbf{68.90}&      0.880\\
4.&dm-cprrdf&$(5,119.31)$&&wtMaj(cls))&&$1$&                                     88.15& 1.63& 68.23&      0.886\\
5.&dm-cprrdf&$(5,119.31)$&&wtMaj(cls))&&$1-\epsilon$&                            88.15& \textbf{1.60}& 68.23&      0.887\\
\noalign{\smallskip}\hline\noalign{\smallskip}
1.&dm-oprdf&$(2,119.31)$&m-estim.&wtMaj(cls)&&$\left( \frac{1-\epsilon}{\epsilon} \right) ^2$&   \textbf{95.64}& \textbf{2.53}& \textbf{77.73}&  \textbf{0.899}\\
2.&dm-cprrdf&$(5,119.31)$&m-estim.&wtMaj(cls)&&$1$&                                              94.86& 2.70& 70.77&   0.859\\
3.&dm-cprrdf&$(5,119.31)$&m-estim.&wtMaj(cls)&&$1-\epsilon$&                                     94.86& 2.63& 70.77&   0.860\\
4.&dm-uprdf&$(5,119.31)$&m-estim.&wtMaj(cls)&&$1$&                                               94.52& 2.70& 70.13&   0.870\\
5.&dm-uprdf&$(5,119.31)$&m-estim.&wtMaj(cls)&&$1-\epsilon$&                                      94.52& 2.67& 70.13&   0.873\\
\noalign{\smallskip}\hline\noalign{\smallskip}
*&rpart&&&&&                              &48.34&0.10&72.00&0.903\\
\noalign{\smallskip}\hline
\end{tabular}
\end{scriptsize}
\end{table*}